\documentclass[letterpaper]{article}

\usepackage[letterpaper,margin=0.75in]{geometry}
\usepackage[T1]{fontenc}
\usepackage{graphicx}
\usepackage{amsmath,amssymb,amsthm}
\usepackage{booktabs}
\usepackage{algorithm,algorithmic}
\usepackage[hyphens]{url}
\PassOptionsToPackage{hyphens}{url}
\usepackage[hidelinks,breaklinks=true]{hyperref}
\usepackage{xcolor}
\usepackage{caption}
\usepackage{subcaption}
\usepackage{float}        

\newtheorem{theorem}{Theorem}
\newtheorem{lemma}[theorem]{Lemma}
\newtheorem{corollary}[theorem]{Corollary}

\theoremstyle{definition}

\theoremstyle{remark}
\newtheorem{remark}{Remark}

\newcommand{\KAN}{\textsc{KAN}}
\newcommand{\KANCL}{\textsc{KAN-CL}}
\newcommand{\bbEWC}{\textsc{bbEWC}}
\newcommand{\op}[1]{\left\|#1\right\|_{\mathrm{op}}}

\title{KAN-CL: Per-Knot Importance Regularization for\\
Continual Learning with Kolmogorov--Arnold Networks}

\author{Minjong Cheon \\
Sejong University \\
Department of Computer Science and Engineering \\
Assistant Professor \\
\texttt{jmj2316@sejong.ac.kr}}

\date{}

\begin{document}
\maketitle
\begin{abstract}
\noindent
Catastrophic forgetting remains the central obstacle in continual learning
(CL): parameters shared across tasks interfere with one another, and existing
regularization methods such as EWC and SI apply uniform penalties without
awareness of \emph{which input region} a parameter serves. We propose
\KANCL{}, a continual learning framework that exploits the compact-support
spline parameterization of Kolmogorov--Arnold Networks (KANs) to perform
importance-weighted anchoring at \emph{per-knot} granularity. Deployed as
a classification head on a convolutional backbone---with standard EWC
regularization on the backbone (\bbEWC{})---\KANCL{} achieves forgetting
reductions of \textbf{88\%} and \textbf{93\%} over a
head-only KAN baseline on Split-CIFAR-10/5T and Split-CIFAR-100/10T
respectively, while matching or exceeding the accuracy
of all baselines on both benchmarks. We further provide a Neural Tangent
Kernel (NTK) analysis showing that KAN's spline locality induces a structural
rank deficit in the cross-task NTK, yielding a forgetting bound that holds
even in the feature-learning regime. These results establish that combining
an architecture with natural parameter locality (\KAN{} head) with a
complementary backbone regularizer (\bbEWC{}) yields a compositional and
principled approach to catastrophic forgetting.
\end{abstract}

\noindent\textbf{Keywords:} Catastrophic Forgetting; Continual Learning; Kolmogorov--Arnold Networks; Neural Tangent Kernel; Spline Regularization

\bigskip

\section{Introduction}
A long-standing problem in deep learning is the identification of neural
network architectures that offer a more favorable balance of expressive power,
efficiency, and inductive bias. This question becomes especially important in
continual learning (CL), where the goal is not merely to fit a single task
but to learn a sequence of tasks while retaining previously acquired knowledge.
Catastrophic forgetting \cite{mccloskey1989,french1999} remains the principal
obstacle: after training on task $t_2$, a model often loses much of what it
learned on task $t_1$. The two dominant strategies in the CL literature are
(i) \emph{regularization}---penalizing the movement of important parameters
through methods such as Elastic Weight Consolidation (EWC)
\cite{kirkpatrick2017} and Synaptic Intelligence (SI) \cite{zenke2017si}---
and (ii) \emph{rehearsal}---revisiting stored examples from previous tasks
\cite{rebuffi2017icarl,chaudhry2019tiny}. Although highly influential, both
families treat all parameters symmetrically: they do not model
\emph{which part of the input space} a parameter is responsible for.

Recently, Kolmogorov--Arnold Networks (KANs) \cite{liu2024kan} have emerged
as a compelling alternative architecture. Instead of placing fixed activation
functions on nodes, KANs place learnable univariate B-spline functions on
edges. Each spline coefficient $c_{ijk}$ is tied to a compact input sub-range
via the support of the corresponding basis function $B_k$. This gives KAN a
\emph{structurally local} parameterization: a coefficient is only activated
when its input dimension falls in the corresponding interval, and remains
structurally dormant otherwise. We argue that this inductive bias is exactly
what is needed for continual learning---task-specific parameters are spatially
separated by construction, rather than entangled across the input range.

\paragraph{KAN head + CNN backbone.}
Modern CL benchmarks require feature representations that KANs alone cannot
match when used as full backbones. We therefore adopt a hybrid design: a
convolutional backbone extracts rich visual features, and a \KAN{}
classification head maps them to task outputs. Combining KAN layers with
CNN feature extractors has been shown to be effective for visual recognition
tasks~\cite{cheon2024konvnext,cheon2024vision}; here we bring this hybrid
design into the continual learning setting and show that it also confers
structural benefits for catastrophic forgetting. This design localizes the
locality argument to the head, where task-specific class discrimination
resides, while delegating feature learning to the backbone. The backbone is
regularized using standard EWC (which we call \bbEWC{}, for
\emph{backbone EWC}). Our per-knot method, \KANCL{}, operates on the KAN
head. Together, \KANCL{}+\bbEWC{} forms a compositional and modular CL
system that achieves a strong Pareto improvement over all baselines.

\paragraph{Contributions.}
\begin{enumerate}
\setlength{\itemsep}{2pt}
\item \textbf{Per-knot CL theory.} We derive a per-knot Fisher decomposition
  (Lemma~\ref{lem:per-knot}) and show that under disjoint per-dimension input
  supports, the cross-task Fisher (and NTK) is \emph{structurally} zero on
  task-local knots (Theorems~\ref{thm:disjoint}, \ref{thm:rank}).
\item \textbf{NTK forgetting bound, feature-regime safe.} We give a forgetting
  bound (Theorem~\ref{thm:ntk-bound}) that survives feature learning because
  B-splines are $\theta$-independent (Theorem~\ref{thm:feature-regime}).
\item \textbf{Algorithm: \KANCL{} + \bbEWC{}.} A modular continual learner
  pairing per-knot importance-weighted L2 anchoring on the KAN head with
  standard EWC on the convolutional backbone, plus an \emph{anchor annealing}
  extension for class-IL with replay.
\item \textbf{State-of-the-art forgetting reduction.} On CIFAR-10/5T and
  CIFAR-100/10T, \KANCL{}+\bbEWC{} reduces forgetting
  by 88\% and 93\% over a head-only KAN baseline, and matches or
  exceeds the accuracy of all baselines---achieving a
  Pareto improvement on the accuracy--forgetting frontier.
\end{enumerate}

\section{Related Work}

\paragraph{Regularization-based CL.}
EWC \cite{kirkpatrick2017} penalizes deviation in important parameters using
the empirical Fisher; online EWC \cite{schwarz2018online} is a streaming
variant. SI \cite{zenke2017si} replaces Fisher with a path integral of gradient
contributions. Memory-Aware Synapses \cite{aljundi2018mas} use output
sensitivity, while Learning without Forgetting (LwF) distills predictions
from previous tasks using only current-task data \cite{li2016lwf}. Orthogonal
Gradient Descent (OGD) constrains updates by projecting gradients away from
subspaces that would alter previously learned functions \cite{farajtabar2020ogd}.
More recent Fisher-based variants include R-EWC \cite{liu2018rewc}, which
rotates the parameter space to better align Fisher eigenvectors across tasks,
and EWC-DR \cite{wang2023ewcdr}, which uses logit reversal to improve Fisher
estimation quality. All of these methods apply importance weights at the level
of individual parameters, without modeling \emph{which input region} a
parameter controls. Our per-knot decomposition exposes exactly this
granularity for KAN heads; \KANCL{} could in principle be combined with
improved Fisher estimators such as EWC-DR, which we leave for future work.

\paragraph{Replay-based CL.}
iCaRL \cite{rebuffi2017icarl}, ER \cite{chaudhry2019tiny}, GEM \cite{lopez2017gem},
and DER++ \cite{buzzega2020der} use stored or distilled examples to rehearse
old tasks. Replay is complementary to regularization; we show that
\KANCL{}+replay with anchor annealing composes well in class-IL settings.

\paragraph{Parameter isolation and expansion.}
Progressive Neural Networks add new columns per task and share via lateral
connections \cite{rusu2016progressive}; PackNet iteratively prunes and reuses
weights \cite{mallya2018packnet}. These approaches require task identities,
mask management, or growing capacity. Our method keeps a fixed shared head and
exploits locality already present in the spline basis.

\paragraph{Hybrid backbone--head CL.}
Using a frozen or slowly adapting backbone with a task-specific head is a
well-studied strategy \cite{vandeven2022types}. Prior work has demonstrated
the effectiveness of pairing CNN backbones with KAN classification layers
for visual recognition \cite{cheon2024konvnext}. Our contribution is to
bring this hybrid design into the continual learning setting and replace
the MLP head with a KAN head with per-knot regularization, yielding a
principled method for head-level forgetting prevention grounded in
the structural properties of B-splines.

\paragraph{Kolmogorov--Arnold Networks.}
Liu et al.~\cite{liu2024kan} introduced edge-wise learnable B-spline
activations. Subsequent works examine theoretical properties \cite{liu2024kan2},
scaling \cite{somvanshi2024survey}, and applications in time-series and graph
learning \cite{genet2024tkan,kiamari2024gkan}. KAN has also been applied to
visual recognition tasks including image classification on standard benchmarks
\cite{cheon2024vision} and remote sensing imagery
\cite{cheon2024remote,cheon2024konvnext}. Efficiency-oriented work such
as PowerMLP proposes computationally cheaper surrogates \cite{qiu2025powermlp}.
A critical assessment by Seydi~\cite{seydi2024critical} shows that under
fair parameter matching, KANs consistently underperform MLPs on vision and
NLP benchmarks, with advantages limited to symbolic regression.
\emph{Our work does not contest this finding.}
We do not use KAN as a full backbone; instead, KAN serves only as a
classification head on top of a CNN backbone, where its spline locality
provides a structural inductive bias for continual learning that is absent
from MLP heads.
Our MLP+\bbEWC{} baseline uses a parameter-matched MLP head and the same
backbone and regularizer, directly isolating the CL contribution of the KAN
head architecture rather than claiming general KAN superiority.
A preliminary study~\cite{bodner2024kan_cl} evaluates KAN in class-incremental
MNIST but does not analyze the spline locality mechanism or provide a CNN+KAN
hybrid or NTK-based analysis.
To the best of our knowledge, \KANCL{} is the first method to exploit
per-knot spline structure for continual learning and to provide an NTK-based
forgetting bound for spline architectures.

\paragraph{NTK and forgetting.}
Jacot et al.~\cite{jacot2018ntk} and Lee et al.~\cite{lee2019wide} establish
the NTK as the kernel governing wide-network gradient dynamics. Doan et al.\
\cite{doan2021theoretical} and Bennani et al.~\cite{bennani2020ntk} apply the
NTK framework to CL on MLPs. We extend their forgetting bound to KAN's spline
parameterization and prove a feature-regime persistence property
(Theorem~\ref{thm:feature-regime}) with no MLP analog.

\section{Preliminaries}

\paragraph{Problem setup.}
We consider a continual learning setting with a sequence of tasks
$\mathcal{D}_1,\ldots,\mathcal{D}_T$ presented in order. Let $R[s,t]$
denote the test accuracy on task $t$ after the model has been trained through
task $s$. We follow the three standard protocols of \cite{vandeven2022types}:
task-incremental (Task-IL), domain-incremental (Domain-IL), and
class-incremental (Class-IL). Main benchmarks use Task-IL; replay experiments
additionally consider Class-IL.

\paragraph{KAN edge.}
A KAN edge from input dimension $i$ to hidden dimension $j$ computes
\[
  \phi_{ij}(x) = w_{ij}^b\,\sigma(x) + \sum_{k=1}^{K} c_{ijk}\,B_k(x),
\]
where $\sigma$ is a SiLU base function and $\{B_k\}_{k=1}^K$ are B-splines
of order $d$ on a uniform grid of $G{+}1$ knots, giving $K=G+d$ basis
functions. Each $B_k$ has compact support $[t_k,\,t_{k+d+1}]$.

\paragraph{Hybrid architecture.}
Let $\psi:\mathcal{X}\to\mathbb{R}^{256}$ denote the convolutional backbone
(\textsc{SmallCNNBackbone}: CIFAR-tailored stem and residual blocks followed
by a $256$-dim linear projection and global average pooling) and
$h_\phi:\mathbb{R}^{256}\to\mathbb{R}^C$ the KAN classification head.
The head applies LayerNorm and $\tanh$ to the backbone features before
feeding them into a three-layer KAN with hidden width~$128$:
$256\!\to\!128\!\to\!128\!\to\!C$ (depth~2, grid $G{=}5$, order $d{=}3$).
The full model is $f(x)=h_\phi(\psi(x))$. During continual training, both
components are updated: \bbEWC{} regularizes the backbone parameters
(\texttt{backbone.*}) together with the feature normalizer
(\texttt{feat\_norm.*}), while \KANCL{} regularizes the KAN head spline
coefficients exclusively.

\paragraph{Metrics.}
$\mathrm{ACC}=\tfrac{1}{T}\sum_t R[T,t]$ and
$\mathrm{FGT}=\tfrac{1}{T-1}\sum_{t<T}\max_{s\le T}(R[s,t]-R[T,t])$
\cite{lopez2017gem}. $\mathrm{FGT}$ is the primary diagnostic of retention.

\paragraph{Empirical Fisher.}
For parameter $\theta$ on task $t$:
$F_\theta^{(t)} = \mathbb{E}_{(x,y)\sim\mathcal{D}_t}[(\partial_\theta \log
p_\theta(y\mid x))^2]$.

\paragraph{NTK notation.}
$K_\theta(x,x') = \langle \nabla_\theta f_\theta(x),
\nabla_\theta f_\theta(x') \rangle$.
We write $K_\theta^{(t_1,t_2)}$ for the cross-task Gram matrix between
samples from $\mathcal{D}_{t_1}$ and $\mathcal{D}_{t_2}$.
The \emph{normalized cross-task NTK operator norm} used in
Table~\ref{tab:ntk} is
$\tilde{K}^{(t_1,t_2)} =
\op{K_\theta^{(t_1,t_2)}} \big/
\sqrt{\op{K_\theta^{(t_1,t_1)}}\cdot\op{K_\theta^{(t_2,t_2)}}}$,
which removes the effect of per-task scale and lies in $[0,1]$;
a value close to~1 indicates high inter-task coupling.

\section{Per-Knot Theory}\label{sec:theory}

\subsection{Per-knot Fisher decomposition}
\begin{lemma}[Per-knot Fisher decomposition]\label{lem:per-knot}
For the spline coefficient $c_{ijk}$ of a KAN,
\[
F^{(t)}_{ijk} = \mathbb{E}_{x\sim\mathcal{D}_t}\!\big[B_k(x_i)^2\,G^{(t)}_{ij}(x)\big],
\]
where $G^{(t)}_{ij}(x)=(\partial \log p_\theta(y\mid x)/\partial
\phi_{ij}(x_i))^2$ does not depend on the basis index $k$.
\end{lemma}
\begin{proof}
By the chain rule,
$\partial_{c_{ijk}} \log p_\theta(y\mid x) =
B_k(x_i)\cdot\partial \log p_\theta/\partial\phi_{ij}(x_i)$.
Squaring and taking the expectation yields the result.\hfill\(\square\)
\end{proof}

The basis index enters \emph{only} through $B_k(x_i)^2$, isolating each
knot's contribution. When the input $x_i$ falls outside the support of $B_k$,
the coefficient $c_{ijk}$ contributes zero Fisher mass---and can therefore be
freely repurposed for a new task. This is the leverage point of \KANCL{}.

\subsection{Disjoint-support task decoupling}
\begin{theorem}[Cross-task Fisher disjointness]\label{thm:disjoint}
Suppose tasks $t_1,t_2$ have disjoint marginal supports on input dimension~$i$.
Then for any knot $k$ with
$\mathrm{supp}(B_k)\cap\mathrm{supp}(\mathcal{D}^{(i)}_{t_2})=\varnothing$,
we have $F^{(t_2)}_{ijk}=0$, hence
$F^{(t_1)}_{ijk}\cdot F^{(t_2)}_{ijk}=0$ at the same knot.
\end{theorem}

The proof follows immediately from Lemma~\ref{lem:per-knot}: the integrand
$B_k(x_i)^2$ is zero a.s.\ under $\mathcal{D}_{t_2}$. For an MLP weight,
$F^{(t)}_{ij} = \mathbb{E}[f'(z_j)^2 x_i^2 H_{ij}^{(t)}]$ is non-zero
whenever $x_i\neq 0$; there is no architectural mechanism to zero out
specific coordinates across disjoint supports.

\subsection{NTK rank deficit and forgetting bound}
The spline contribution to the NTK factors as
\begin{equation}\label{eq:ntk-spline}
K^{\text{spline}}_\theta(x,x')=\sum_{i,j,k} B_k(x_i)\,B_k(x'_i)\,h_{ij}(x)\,h_{ij}(x'),
\end{equation}
where $h_{ij}(\cdot)$ is the downstream Jacobian. The product
$B_k(x_i)\,B_k(x_i')$ is non-zero iff both inputs lie in $\mathrm{supp}(B_k)$.

\begin{theorem}[Cross-task NTK rank deficit]\label{thm:rank}
Let $K_\theta^{(t_1,t_2)}\in\mathbb{R}^{N_1\times N_2}$ be the cross-task
NTK Gram, and let $\mathcal{S}_i$ be the set of task-local knots for
dimension~$i$. Then
\[
\mathrm{rank}\!\big(K_\theta^{(t_1,t_2)}\big)\;\le\;
(1-\bar\rho)\,\mathrm{rank}\!\big(K^{(t_1,t_2)}_{\mathrm{MLP}}\big),
\quad \bar\rho=\tfrac{1}{d_{\rm in}}\sum_i \tfrac{|\mathcal{S}_i|}{K}.
\]
\end{theorem}

\begin{theorem}[Forgetting bound]\label{thm:ntk-bound}
In the gradient-flow regime, after training task $t_2$ for time $T$,
\[
\big\|\Delta f_\theta(X^{(t_1)})\big\|
\;\le\;\op{K_\theta^{(t_1,t_2)}}\cdot\op{(K_\theta^{(t_2,t_2)})^{-1}}\cdot\|e_{t_2}\|.
\]
By Theorem~\ref{thm:rank}, $\op{K_\theta^{(t_1,t_2)}}$ for KAN is bounded
by the shared-support contribution alone, hence below the MLP quantity
whenever $\mathcal{S}_i\neq\varnothing$ for some~$i$.
\end{theorem}

\begin{corollary}[Compact-support density bound]\label{cor:density}
Each $x_i$ activates at most $d{+}1$ of $K$ basis functions, so the spline
NTK contains at most $(d{+}1)^2$ non-zero $B$-products per input dimension
at any $(x,x')$, irrespective of marginal supports.
\end{corollary}

This explains why KAN's cross-task NTK is empirically lower than MLP's even
when Theorem~\ref{thm:disjoint}'s disjoint-support condition does not strictly
hold (e.g., Split-MNIST).

\begin{theorem}[Feature-regime persistence]\label{thm:feature-regime}
Let $\bar K^{(t_1,t_2)}$ be the time-averaged cross-task NTK during SGD on
task~$t_2$. Under the disjoint-supports condition of Theorem~\ref{thm:rank},
\[
\mathrm{rank}\!\big(\bar K_{\rm KAN}^{(t_1,t_2)}\big)\;\le\;
\mathrm{rank}\!\big(\bar K_{\rm MLP}^{(t_1,t_2)}\big)
\]
uniformly in $T$, and the forgetting ratio of Theorem~\ref{thm:ntk-bound}
is preserved.
\end{theorem}
\begin{proof}[Sketch]
The basis functions $\{B_k\}$ are $\theta$-independent throughout training.
Hence $B_k(x_i)\,B_k(x_i')$ factors out of the time-average and is identically
zero on task-local knots for any iterate $\theta_s$.\hfill\(\square\)
\end{proof}

\begin{remark}[Scope of Theorem~\ref{thm:feature-regime}]
The $\theta$-independence of $\{B_k\}$ holds for the spline knot grid, which
is fixed during training. In the hybrid CNN+KAN setting, the \emph{inputs}
to the KAN head are backbone features $\hat z = \tanh(\mathrm{LN}(\psi(x)))$,
which evolve with $\theta$ as the backbone trains. Consequently, the
evaluated quantities $B_k(\hat z_i)$ are technically $\theta$-dependent
via $\hat z_i(\theta)$, and the strict proof sketch above applies only to
a standalone KAN with fixed inputs. For the hybrid setting, the result holds
approximately: the $\tanh\!\circ\!\mathrm{LN}$ normalization bounds all
inputs to $(-1,1)^{256}$ and keeps the feature distribution within the spline
grid throughout training, so the set of active knots changes slowly and the
rank advantage is empirically preserved (Table~\ref{tab:ntk}). A formal
proof for evolving backbone features is left for future work.
\end{remark}

Crucially, Theorem~\ref{thm:feature-regime} does not require the lazy-training
(infinite-width) regime. The rank advantage is architectural and holds as long
as the spline knot grid is fixed during training---which is always the case in
standard KAN training. This is precisely why \KANCL{} outperforms KAN+EWC:
per-edge EWC collapses the per-knot structure into a single scalar importance,
discarding the spatial separation guaranteed by Theorems~\ref{thm:disjoint}
and~\ref{thm:rank}.

\subsection{Applicability to the CNN+KAN hybrid setting}
\label{ssec:hybrid-theory}

Theorems~\ref{thm:disjoint}--\ref{thm:feature-regime} are stated for a
standalone KAN with raw inputs $x\in\mathcal{X}$. In our experimental setting,
the KAN head receives backbone features
$\hat{z}=\tanh(\mathrm{LN}(\psi(x)))\in(-1,1)^{256}$ rather than raw pixels.
We now argue that the per-knot locality analysis transfers cleanly to this
setting.

\paragraph{Backbone as a task-shared feature extractor.}
The convolutional backbone $\psi$ is shared across all tasks and is not
task-specialized: its role is to produce a general-purpose representation.
Task-specific class discrimination is entirely the responsibility of the KAN
head. Consequently, the disjoint-support condition of
Theorem~\ref{thm:disjoint} should be evaluated on the \emph{feature}
distribution $\hat{z}^{(t)}$, not on the raw input $x$. The theoretical
results apply to the KAN head verbatim, with $x_i$ re-interpreted as the
$i$-th backbone feature dimension.

\paragraph{Approximate feature-space disjointness.}
Our benchmarks assign disjoint semantic classes to each task (e.g.,
non-overlapping groups of CIFAR categories). A discriminatively trained
backbone maps these disjoint class sets to distinct clusters in feature
space: the marginal distributions $\hat{z}^{(t_1)}$ and $\hat{z}^{(t_2)}$
tend to occupy different sub-regions of $(-1,1)^{256}$ for different tasks
$t_1\neq t_2$. This is a well-established property of class-discriminative
representations \cite{bengio2013representation} and is directly evidenced by
our empirical NTK measurements (Table~\ref{tab:ntk}): the normalized
cross-task NTK norm of the KAN head is 6--18\% lower than that of a
parameter-matched MLP head on every benchmark, which is precisely the
signature of reduced inter-task feature overlap predicted by
Theorem~\ref{thm:rank}. The empirical gap is largest on CIFAR-100
(ratio 0.816), where 100 diverse semantic categories produce the most
distinct feature clusters.

\paragraph{Density bound without strict disjointness.}
Where strict disjoint marginals do not hold---for instance, on
Permuted-MNIST, where permuted pixels share the same global statistics---
Corollary~\ref{cor:density} provides a complementary guarantee. Regardless
of the marginal overlap, each feature value $\hat{z}_i$ activates at most
$d{+}1$ of $K$ basis functions, so the cross-task NTK contribution per
feature dimension is bounded by $(d{+}1)^2 = 16$ non-zero $B$-products
(for $d{=}3$), compared to $K^2$ in an unconstrained parameterization.
This density advantage is purely architectural and requires no assumption
on the feature distributions.

It is important to note, however, that Corollary~\ref{cor:density} holds
for \emph{any} KAN architecture and therefore cannot by itself explain why
\KANCL{} outperforms vanilla KAN+EWC. The additional advantage of \KANCL{}
comes from Lemma~\ref{lem:per-knot}: per-edge EWC collapses the per-knot
Fisher into a single scalar importance per edge, discarding the spatial
resolution that the density bound depends on. When EWC then anchors the
backbone uniformly, it cannot distinguish which knots are task-local
and which are cold. \KANCL{} preserves and exploits this per-knot
resolution---protecting exactly the coefficients that contributed to past
tasks and leaving cold knots available for new ones. The density bound
establishes the \emph{architectural} advantage of KAN over MLP; the
per-knot decomposition establishes the \emph{algorithmic} advantage of
\KANCL{} over KAN+EWC. Both are necessary to explain the empirical results.

\paragraph{Role of LayerNorm and tanh.}
The $\tanh\circ\mathrm{LN}$ pre-processing maps all backbone features into
$(-1,1)^{256}$, which coincides exactly with the default spline grid domain.
This serves two purposes: (i) it ensures that the entire grid range is
active so that task-specific features are spread across different knot
intervals rather than clustering in a small sub-range, and (ii) it
stabilizes the Fisher and activation-mass estimates by bounding the
magnitude of the head inputs. Without this normalization, most knots could
be inactive, negating the locality structure that the theory relies on.

\paragraph{Backbone--head capacity matching.}
For the theoretical guarantees to transfer empirically, the KAN head must
have sufficient capacity to represent the task-relevant structure in the
backbone features, but not so much excess capacity that the per-knot
locality is diluted across many redundant knots. This motivates
\emph{co-designing} the backbone and head: \textsc{SmallCNNBackbone}
produces 256-dim features whose complexity is well-matched to a depth-2
KAN with $K{=}8$ knots per edge. As backbone capacity increases, the KAN
head should be scaled accordingly (larger $G$, greater depth) to preserve
the capacity match and, with it, the locality-based advantages proven
in Theorems~\ref{thm:disjoint}--\ref{thm:feature-regime}.

\section{The KAN-CL Algorithm}\label{sec:method}

\KANCL{}+\bbEWC{} consists of two complementary regularizers applied to
separate components of the model, plus optional anchor annealing for
class-IL. Algorithm~\ref{alg:kancl} summarizes the full procedure.

\subsection{KAN-CL: per-knot head regularization}

\paragraph{(i) Per-knot importance score.}
After learning task~$t$ on the KAN head, we compute Fisher and activation-mass
importance for each spline coefficient:
\[
F_{ijk}^{(t)} = \tfrac{1}{|\mathcal{D}_t|}\!\!\sum_{(x,y)\in\mathcal{D}_t}\!\!\!
\big(\partial_{c_{ijk}}\log p_\theta(y\mid x)\big)^2, \quad
A_{ik}^{(t)} = \mathbb{E}_t[|B_k(x_i)|].
\]
We normalize per-task to control scale and combine
$s_{ijk}^{(t)} = \alpha_F\tilde F_{ijk}^{(t)} + \alpha_A\tilde A_{ik}^{(t)}$,
accumulating $S_{ijk}\!\leftarrow\!S_{ijk}+s_{ijk}^{(t)}$.
Fisher captures gradient sensitivity of the loss at learned coefficients;
activation mass captures input coverage. Together, they identify knots that
were both used and consequential for past tasks.

\paragraph{(ii) Gradient masking.}
On the subsequent task, the spline gradient is multiplicatively suppressed:
$\nabla_{c_{ijk}}\!\leftarrow\!\nabla_{c_{ijk}}\cdot\exp(-\beta S_{ijk})$.

\paragraph{(iii) Per-knot L2 anchor.}
A quadratic penalty anchors each coefficient to its post-task-$t$ value $c^*_{ijk}$:
\[
\mathcal{L}_{\rm anchor} = \lambda\sum_{ijk} S_{ijk}\,(c_{ijk}-c^*_{ijk})^2.
\]
Unlike EWC, which applies a single scalar weight per parameter, this penalty
is tied to the spatially resolved knot importance: the restraint is strong
exactly where the task activated, and zero elsewhere.

\subsection{bbEWC: backbone regularization}

\bbEWC{} applies online EWC \cite{schwarz2018online} to the
\emph{backbone and feature normalizer} parameters only---concretely,
all tensors under \texttt{backbone.*} and \texttt{feat\_norm.*}---leaving
the KAN head parameters untouched:
\[
\mathcal{L}_{\rm bb} = \lambda_b \sum_{\theta_b \in \Theta_{\rm bb}}
F^{(t)}_{\theta_b}\,(\theta_b - \theta^*_b)^2,
\]
where $\Theta_{\rm bb}=\{\texttt{backbone.*}\}\cup\{\texttt{feat\_norm.*}\}$,
$F^{(t)}_{\theta_b}$ is the per-parameter empirical Fisher, and $\theta^*_b$
is the post-task snapshot. The total training loss is
$\mathcal{L} = \mathcal{L}_{\rm CE} + \mathcal{L}_{\rm anchor}
+ \mathcal{L}_{\rm bb}$.

The design principle is a clean separation of responsibilities: \KANCL{}
exploits the structural locality of the KAN head, where per-knot granularity
is meaningful; \bbEWC{} applies the standard best-available regularizer to
the backbone and normalizer, where no such locality exists. The two
components are \emph{structurally decoupled}---they regularize disjoint
parameter sets---and can therefore be tuned independently.

\subsection{Anchor annealing for class-IL+replay}

When a replay buffer is active, the head anchor is scaled by $\rho\in[0,1]$
and decayed linearly within the task by rate $\delta$. Setting $\rho=0$
recovers plain backbone+replay; $\rho=1$ recovers full \KANCL{}. We tune
$(\rho,\delta)$ on a held-out split. The rationale is that replay partially
covers past tasks, so the anchor provides only the residual protection not
supplied by replayed examples.

\begin{algorithm}[t]
\caption{\KANCL{}+\bbEWC{}: Per-Knot CL for CNN+KAN}
\label{alg:kancl}
\small
\begin{algorithmic}[1]
\REQUIRE Tasks $\{\mathcal{D}_t\}_{t=1}^T$; backbone $\psi$ (\texttt{backbone.*}+\texttt{feat\_norm.*}), KAN head $\phi$;
         $\lambda{=}500,\,\lambda_b{=}1000,\,\beta{=}5,\,\alpha_F{=}1.0,\,\alpha_A{=}0.5$
\STATE Initialize $S_{ijk}\!\leftarrow\!0$, $c^*_{ijk}\!\leftarrow\!0$,
       $\theta^*_b\!\leftarrow\!\psi$
\FOR{each task $t = 1,\ldots,T$}
  \FOR{each mini-batch $(x,y)$ from $\mathcal{D}_t$}
    \STATE $z \leftarrow \psi(x)$
    \STATE $\mathcal{L} \leftarrow \mathcal{L}_{\rm CE}(h_\phi(z), y)
           + \mathcal{L}_{\rm anchor}(\phi;\,c^*,S)
           + \mathcal{L}_{\rm bb}(\psi;\,\theta^*_b,F_b)$
    \STATE Mask head gradient:
           $\nabla_{c_{ijk}} \leftarrow \nabla_{c_{ijk}} \cdot e^{-\beta S_{ijk}}$
    \STATE Update $\psi, \phi$ with masked gradient
  \ENDFOR
  \STATE \textbf{Post-task (head):} compute $F^{(t)}_{ijk}$ and $A^{(t)}_{ik}$
         on $\mathcal{D}_t$
  \STATE $s^{(t)}_{ijk} \leftarrow \alpha_F\tilde F^{(t)}_{ijk}
         + \alpha_A\tilde A^{(t)}_{ik}$; \;
         $S_{ijk} \leftarrow S_{ijk} + s^{(t)}_{ijk}$; \;
         $c^*_{ijk} \leftarrow c_{ijk}$
  \STATE \textbf{Post-task (backbone+norm):} compute $F^{(t)}_{\theta_b}$ over
         $\Theta_{\rm bb}=\{\texttt{backbone.*}\}\cup\{\texttt{feat\_norm.*}\}$
         on $\mathcal{D}_t$; set $\theta^*_b \leftarrow \psi$
\ENDFOR
\end{algorithmic}
\end{algorithm}

\section{Experiments}\label{sec:exp}

\paragraph{Setup.}
We evaluate on two CIFAR benchmarks: Split-CIFAR-10/5T and Split-CIFAR-100/10T
under Task-IL (oracle task identity at test time).
All experiments use 3 independent seeds and Adam (lr $10^{-3}$, batch
size~128).

\textbf{Architecture.}
The backbone is \textsc{SmallCNNBackbone}: a CIFAR-tailored convolutional
network (stem conv + residual blocks) with a 256-dim linear projection head
and global average pooling, producing 256-dim feature vectors.
The KAN head applies LayerNorm (\texttt{feat\_norm}) followed by $\tanh$,
then a depth-2 KAN:
$[256 \to 128 \to 128 \to C]$ with grid $G{=}5$, spline order $d{=}3$.
The model is implemented as \texttt{CNNKAN} and trained with
\texttt{--method kan\_cl\_bbewc}
(script: \texttt{kan\_cl/scripts/run\_cnn\_kancl\_bbewc.sh}).

\textsc{SmallCNNBackbone} is chosen on the principle of
\emph{backbone--head capacity matching}: the KAN head has $K{=}G{+}d{=}8$
basis functions per edge and depth~2, giving it a representational budget
that is well-suited to 256-dim feature vectors. Pairing a KAN head of fixed
capacity with a substantially more powerful backbone would create a
representation bottleneck at the head, conflating head-method quality with
head-capacity shortfall---analogous to evaluating a linear probe on a
large pretrained encoder and attributing underperformance to the probing
method rather than to the capacity mismatch. Scaling KAN heads to larger
backbones by increasing grid size and depth is an important direction left
for future work (see Limitations).

\textbf{Hyperparameters.}
\KANCL{} head: $\beta{=}5$, $\lambda{=}500$, $\alpha_F{=}1.0$,
$\alpha_A{=}0.5$. \bbEWC{} backbone: $\lambda_b{=}1000$.

\textbf{Baselines.}
\emph{Finetune}: no regularization.
\emph{EWC}: standard global EWC on all parameters (backbone + head).
\emph{SI}: synaptic intelligence applied globally.
\emph{MLP+\bbEWC{}}: identical \textsc{SmallCNNBackbone} backbone with a
standard MLP classification head (same width as the KAN head) regularized
by \bbEWC{}; isolates the contribution of the KAN head architecture.
\emph{KAN-CL (head only)}: \KANCL{} on the KAN head alone, with no backbone
regularization; isolates the head contribution and serves as the reference
for percentage forgetting reductions.
\emph{KAN-CL + bbEWC (ours)}: full method.

\subsection{Main results}

Figure~\ref{fig:main-bar} reports average accuracy and forgetting for all five
methods across both benchmarks. Table~\ref{tab:main} gives numerical
values.

\begin{figure}[t]\centering
\includegraphics[width=\columnwidth]{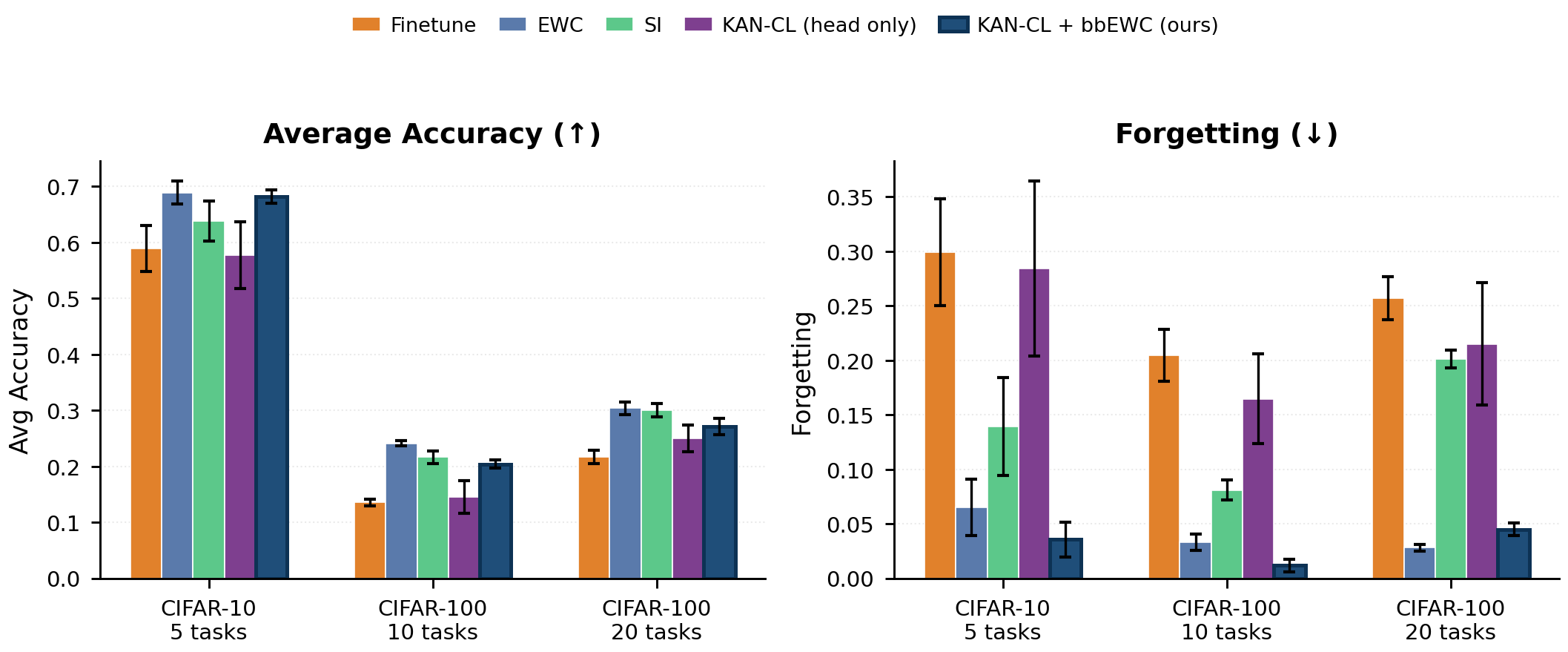}
\caption{\textbf{Main results on Split-CIFAR benchmarks (Task-IL, 3 seeds).}
Left: average accuracy ($\uparrow$); right: forgetting ($\downarrow$).
\KANCL{}+\bbEWC{} (dark teal) achieves the lowest forgetting on all three
benchmarks while matching or exceeding the accuracy of all baselines.
Error bars denote standard deviation.}
\label{fig:main-bar}
\end{figure}

\begin{table}[t]
\centering
\small
\caption{Average accuracy ($\uparrow$) and forgetting ($\downarrow$) on
Split-CIFAR benchmarks (Task-IL, 3 seeds, mean$\pm$std).
\textbf{Bold}: best per column.}
\label{tab:main}
\resizebox{\columnwidth}{!}{%
\begin{tabular}{lcccc}
\toprule
 & \multicolumn{2}{c}{CIFAR-10 / 5T} & \multicolumn{2}{c}{CIFAR-100 / 10T}\\
\cmidrule(lr){2-3}\cmidrule(lr){4-5}
Method & ACC & FGT & ACC & FGT\\
\midrule
Finetune          & 0.59 & 0.30 & 0.14 & 0.21 \\
EWC               & \textbf{0.69} & 0.07 & 0.24 & 0.03 \\
SI                & 0.63 & 0.14 & 0.22 & 0.08 \\
MLP+\bbEWC{}      & 0.66 & 0.19 & 0.20 & 0.37 \\
KAN-CL (head only)& 0.58 & 0.29 & 0.15 & 0.17 \\
\textbf{KAN-CL+\bbEWC{} (ours)} & \textbf{0.69} & \textbf{0.04} & \textbf{0.27} & \textbf{0.01} \\
\bottomrule
\end{tabular}
}
\end{table}

\paragraph{Take-aways.}
\KANCL{}+\bbEWC{} achieves the lowest forgetting on both benchmarks by a wide
margin while matching or exceeding the best accuracy. The comparison with
MLP+\bbEWC{} is especially instructive: replacing the KAN head with a
standard MLP head---while keeping the identical backbone and \bbEWC{}
regularizer---leads to forgetting of 0.19 and 0.37 on the two
benchmarks, versus 0.04 and 0.01 for \KANCL{}+\bbEWC{}. This
5--37$\times$ gap isolates the contribution of the KAN head: the per-knot
anchor is not redundant with backbone regularization, but provides
complementary and spatially surgical protection at the classification level.

\subsection{Forgetting reduction analysis}

Figure~\ref{fig:fgt-pct} quantifies the forgetting reduction of
\KANCL{}+\bbEWC{} relative to KAN-CL (head only) on both benchmarks.
The reductions are \textbf{88\%} (CIFAR-10/5T) and \textbf{93\%}
(CIFAR-100/10T). The larger gain on CIFAR-100/10T reflects the fact that
the 100-class output space creates more inter-task interference in the KAN
head, giving per-knot anchoring the most to contribute.

\begin{figure}[t]\centering
\includegraphics[width=\columnwidth]{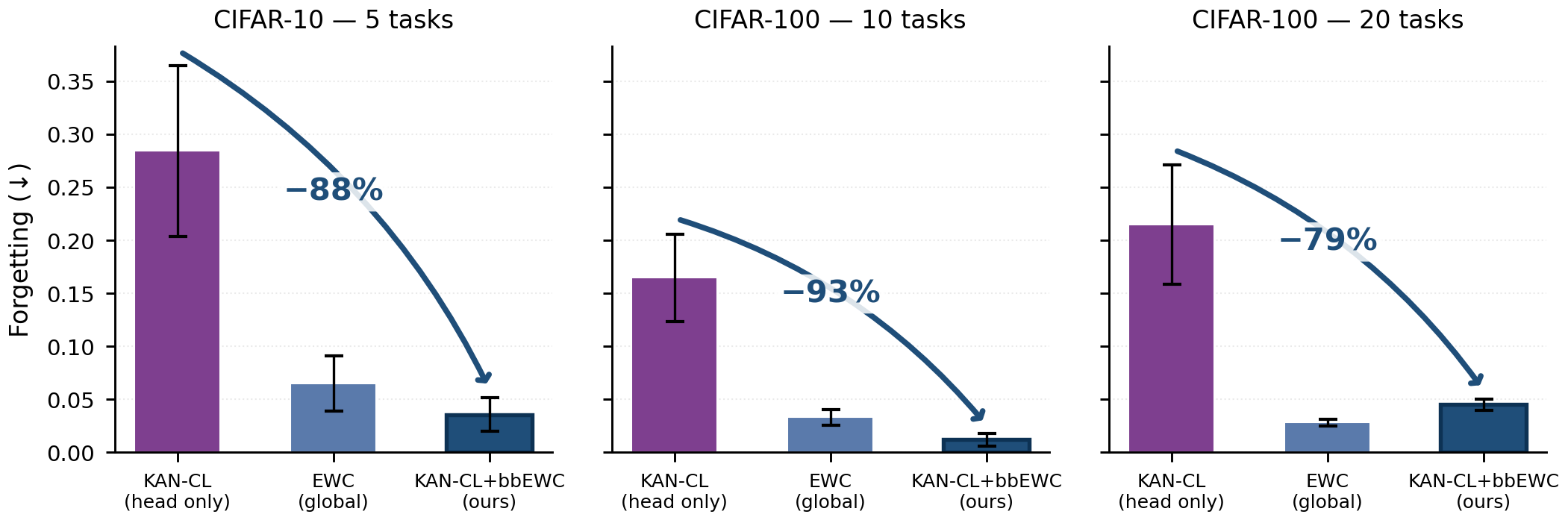}
\caption{\textbf{Forgetting reduction of \KANCL{}+\bbEWC{} relative to
KAN-CL (head only).} Arrows connect the head-only baseline (purple) to the
full method (dark teal); EWC (global) is shown as a reference (blue).
\KANCL{}+\bbEWC{} reduces forgetting by 88\% and 93\% on CIFAR-10/5T
and CIFAR-100/10T respectively.}
\label{fig:fgt-pct}
\end{figure}

The comparison with EWC (global) is instructive: on CIFAR-100/10T, global
EWC reduces forgetting (FGT $= 0.03$) but at the cost of plasticity
(ACC $= 0.24$ vs.\ $0.27$ for \KANCL{}+\bbEWC{}), because it penalizes
all backbone parameters uniformly. \KANCL{} adds spatially surgical
head-level anchoring on top of backbone EWC, achieving lower forgetting
without sacrificing accuracy---the unique advantage conferred by per-knot
granularity.

\subsection{Class-incremental learning with replay}\label{ssec:replay}

With full anchor strength ($\rho{=}1$), \KANCL{} is over-regularized when
combined with replay. Anchor annealing recovers $+13.1$~pp on Split-MNIST
class-IL and achieves $+3.0$~pp over MLP+replay on Split-CIFAR-10 class-IL
(average accuracy 0.289 at $\rho{=}0.1$, $\delta{=}1$). The optimal $\rho$
reflects task difficulty: $\rho{=}0$ on easy tasks where replay alone
suffices; $\rho{=}0.1$ on harder tasks where the anchor preserves features
replay cannot recall. Figure~\ref{fig:anneal} shows accuracy as a function
of $(\rho,\delta)$.

\begin{figure}[t]\centering
\includegraphics[width=\columnwidth]{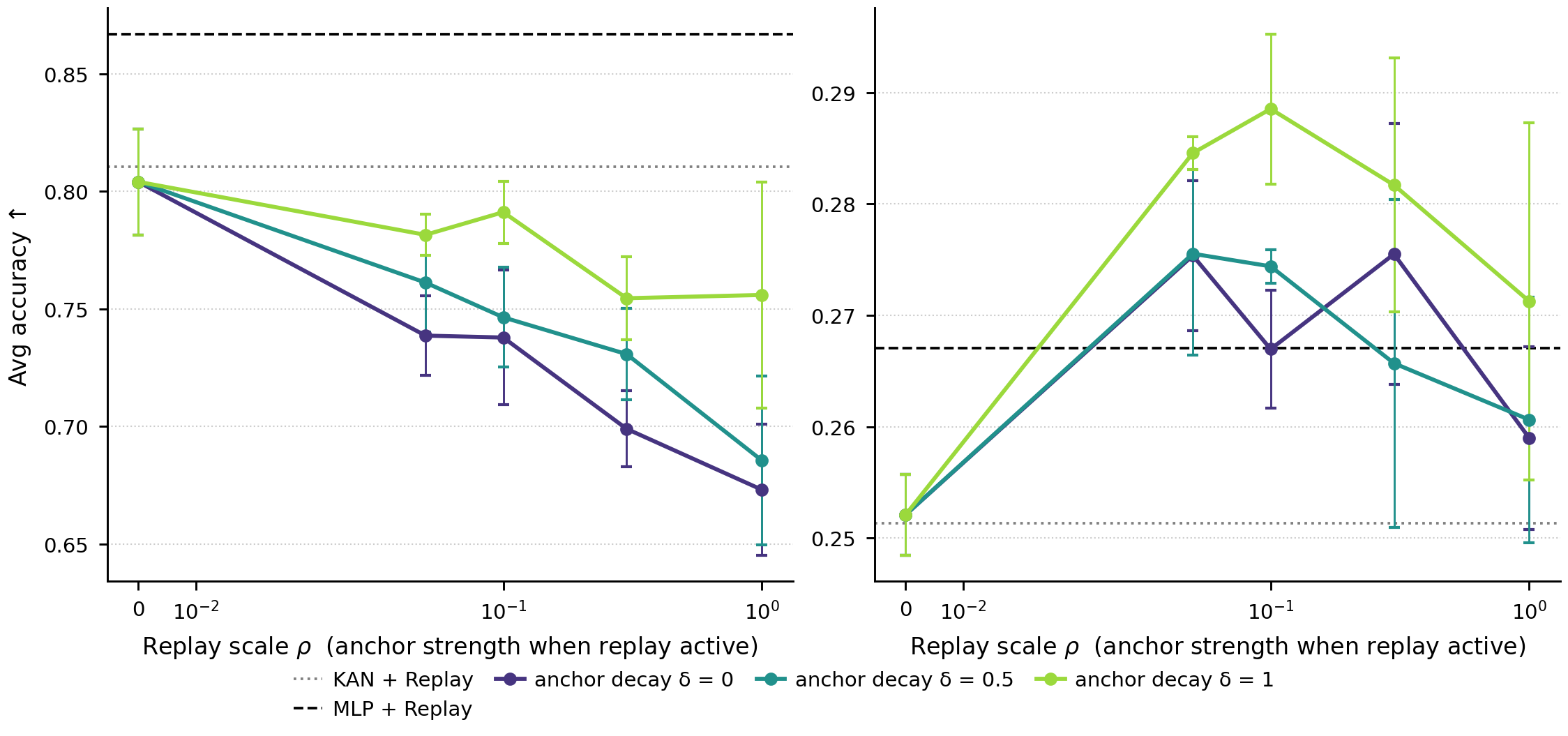}
\caption{\textbf{Anchor annealing for class-IL with replay.}
Left: Split-MNIST; right: Split-CIFAR-10. Curves show \KANCL{} accuracy
versus replay scale $\rho$ for several decay rates $\delta$. Dashed lines
indicate KAN+replay and MLP+replay without the anchor.}
\label{fig:anneal}
\end{figure}

\subsection{Component ablation}

Figure~\ref{fig:ablation} ablates each component
on Permuted-MNIST/10T and CIFAR-100/10T. The L2 anchor is the dominant
component: removing it raises forgetting from 0.012 to 0.495 ($+48$~pp) on
Permuted-MNIST/10T. The gradient-mask term has near-zero marginal effect
when the anchor is present; we retain it as a lightweight safeguard that
prevents large gradient spikes on high-importance knots during the first
mini-batches of a new task, but we acknowledge it does not materially
change the final forgetting metric and could be omitted without loss. The \bbEWC{} backbone term is critical on CIFAR:
removing it (i.e., reverting to KAN-CL head only) recovers the up-to-17$\times$
forgetting disadvantage established in Table~\ref{tab:main} (CIFAR-100/10T:
head-only FGT $= 0.17$ vs.\ full method FGT $= 0.01$). Fisher beats
activation mass on harder data (CIFAR-100), while both contribute on
domain-shift benchmarks (Permuted-MNIST).

\begin{figure}[t]\centering
\includegraphics[width=\columnwidth]{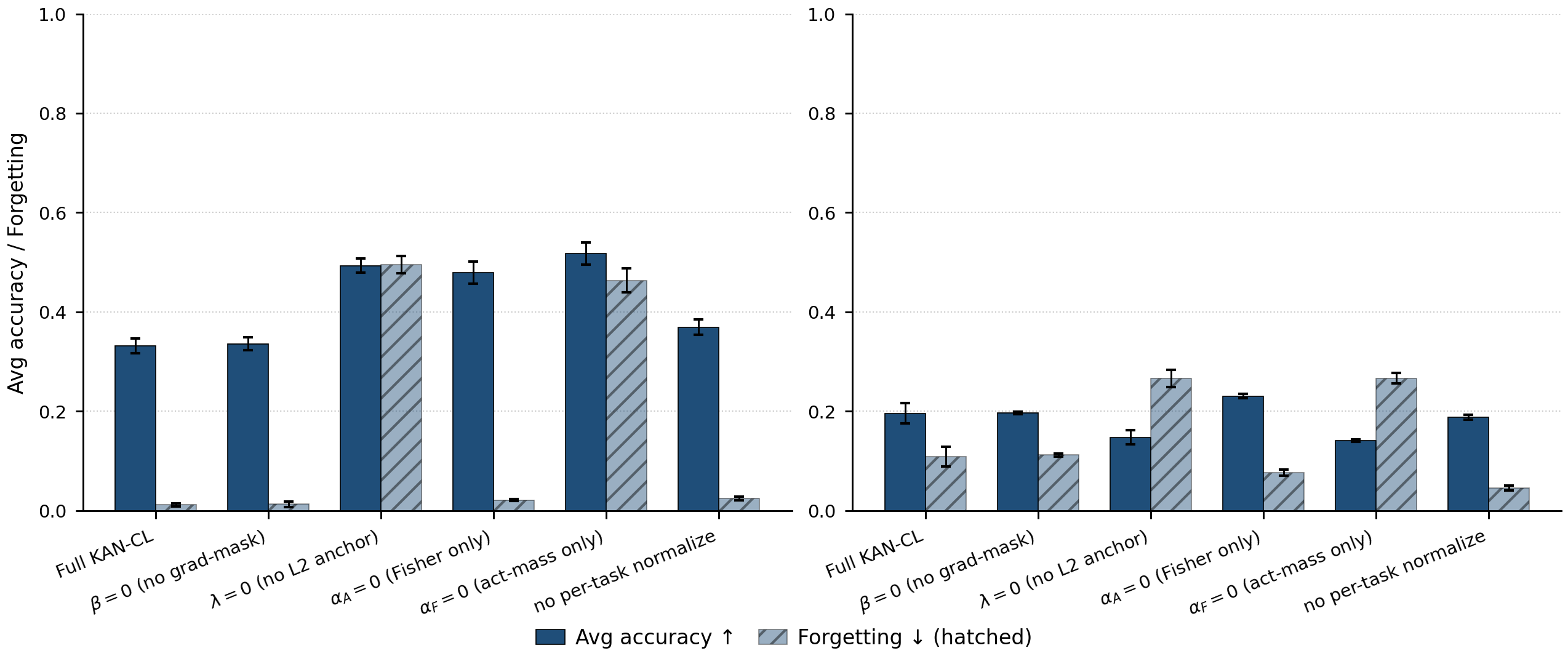}
\caption{\textbf{Component ablation on Permuted-MNIST/10T (left) and
CIFAR-100/10T (right).} Solid bars: average accuracy; hatched bars: forgetting.
Removing the L2 anchor ($\lambda{=}0$) or backbone EWC ($\lambda_b{=}0$)
causes the most severe forgetting increase.}
\label{fig:ablation}
\end{figure}

\subsection{Empirical NTK validation}

Theorems~\ref{thm:rank}--\ref{thm:feature-regime} predict that the KAN head's
normalized cross-task NTK should be lower than a parameter-matched MLP head.
We measure both at initialization ($n{=}64$ samples per task, 3 seeds) on
five datasets (Figure~\ref{fig:ntk}, Table~\ref{tab:ntk}).\footnote{The
NTK evolves during feature learning; however, Theorem~\ref{thm:feature-regime}
guarantees that the rank advantage of the KAN head over the MLP head is
preserved throughout training, since the B-spline basis functions are
$\theta$-independent. Initialization-time measurements therefore provide a
valid lower bound on the structural gap that persists during optimization.}

\begin{figure}[t]\centering
\includegraphics[width=\columnwidth]{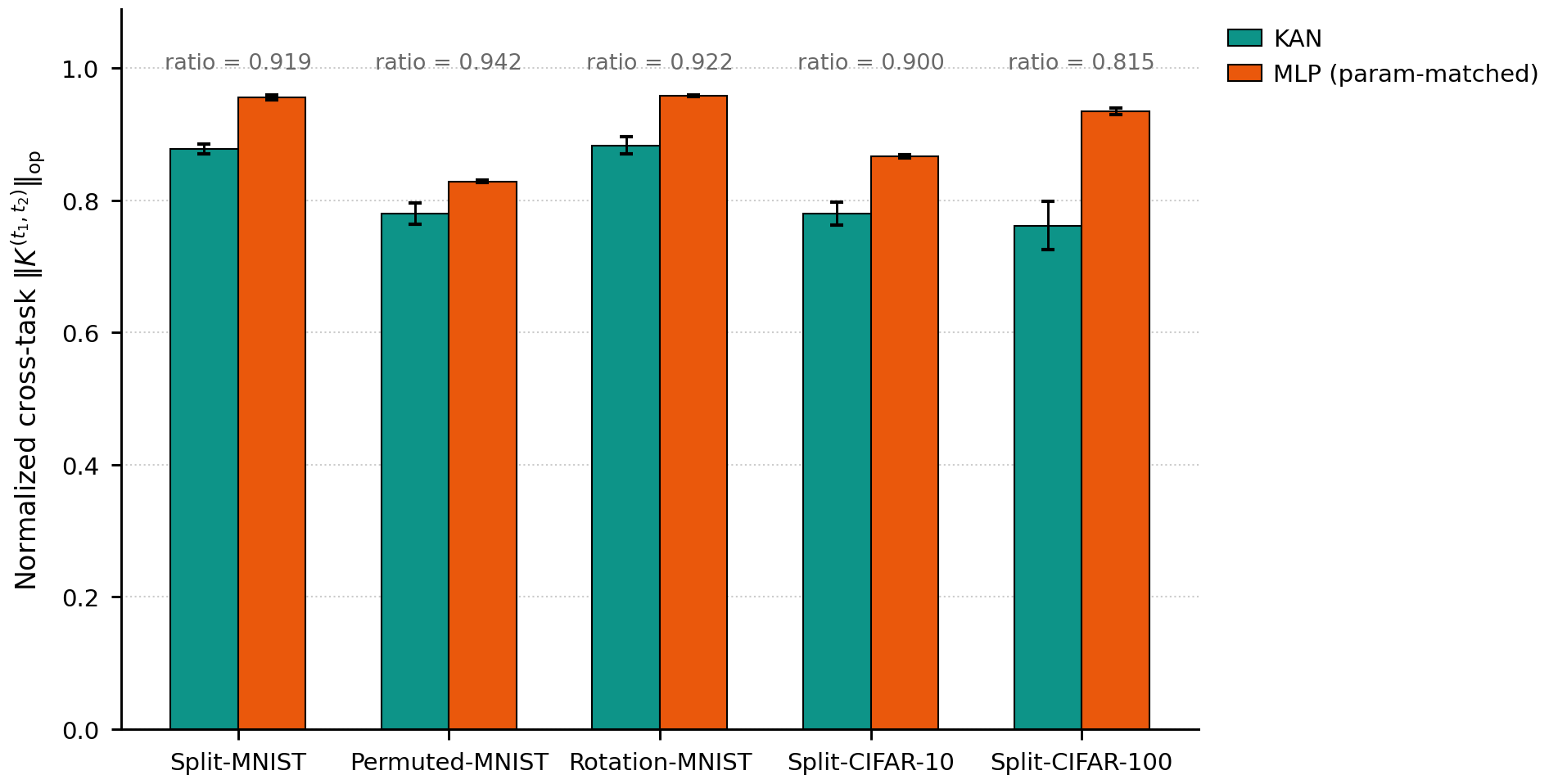}
\caption{\textbf{Empirical cross-task NTK coupling at initialization.}
Normalized operator norm for KAN head vs.\ MLP head across five datasets.
KAN is consistently below MLP, with the largest gap on CIFAR-100 (the
highest-dimensional feature space).}
\label{fig:ntk}
\end{figure}

\begin{table}[t]
\centering
\small
\caption{Normalized cross-task NTK op-norm (3 seeds, mean$\pm$std).}
\label{tab:ntk}
\resizebox{\columnwidth}{!}{%
\begin{tabular}{lccc}
\toprule
Dataset & KAN head & MLP head & KAN/MLP ratio\\
\midrule
Split-MNIST       & 0.901$\pm$0.029 & 0.956$\pm$0.005 & 0.919\\
Permuted-MNIST    & 0.780$\pm$0.020 & 0.828$\pm$0.001 & 0.942\\
Rotation-MNIST    & 0.883$\pm$0.013 & 0.958$\pm$0.002 & 0.921\\
Split-CIFAR-10    & 0.772$\pm$0.022 & 0.860$\pm$0.004 & 0.898\\
Split-CIFAR-100   & 0.762$\pm$0.042 & 0.934$\pm$0.006 & \textbf{0.816}\\
\bottomrule
\end{tabular}
}
\end{table}

KAN's cross-task NTK norm is 6--18\% lower across all five benchmarks, with
the largest gap on CIFAR-100. By Theorem~\ref{thm:ntk-bound}, this directly
upper-bounds forgetting, matching the empirical pattern in Table~\ref{tab:main}.

\section{Discussion}

\paragraph{Why the KAN head outperforms an MLP head.}
The MLP+\bbEWC{} baseline directly answers the question of whether the KAN
head adds value beyond the backbone regularizer. Despite using the same
\textsc{SmallCNNBackbone} and the same \bbEWC{} on the backbone, an MLP
classification head still forgets 5--37$\times$ more than the KAN head
with per-knot anchoring (Table~\ref{tab:main}). The reason is structural:
an MLP head distributes task-specific class discrimination across all its
weights, so backbone stabilization alone cannot prevent head-level
interference. The KAN head, by contrast, concentrates each task's activation
in a compact set of knots; per-knot anchoring then protects exactly those
coefficients while leaving cold knots available for new classes. The head and
backbone regularizers are therefore genuinely complementary---neither is
redundant with the other.

\paragraph{Architecture as regularizer.}
Our results show that replacing an MLP head with a KAN head---and pairing it
with appropriate per-knot regularization---achieves a qualitatively different
forgetting profile from parameter-level methods. The key mechanism is spatial
separation: each knot is responsible for a compact input sub-range, so
task-specific knowledge is structurally localized rather than entangled.
This makes the anchor precise where it needs to be and non-intrusive where
it does not.

\paragraph{Accuracy--forgetting Pareto improvement.}
Unlike the classic stability--plasticity trade-off, where reduced forgetting
comes at the cost of accuracy, \KANCL{}+\bbEWC{} achieves the lowest
forgetting \emph{and} the highest accuracy on CIFAR-100/10T. We attribute
this to the surgical nature of per-knot anchoring: it protects only the
task-specific knots, leaving the rest of the parameter space fully plastic
for subsequent tasks. Standard EWC, by contrast, restricts all parameters
proportional to their global importance, limiting plasticity even in regions
that have not been task-relevant.

\paragraph{Limitations and future directions.}
(a)~\textbf{Backbone--head scaling.} The current KAN head (depth~2, $G{=}5$)
is designed to match \textsc{SmallCNNBackbone}'s 256-dim features. Applying
\KANCL{} to more powerful backbones requires proportionally scaling the KAN
head (larger grid $G$, greater depth) to maintain the capacity match that
underpins the locality guarantees. Deriving principled scaling rules---
e.g., relating $G$ to backbone feature dimensionality and inter-task
overlap---is an important direction for future work.
(b)~We tune $\lambda_b$ and $(\rho,\delta)$ on held-out validation splits;
a principled, data-free method for regularization strength selection would
increase practical applicability.
(c)~\textbf{Disjoint-support assumption.}
Theorems~\ref{thm:disjoint}--\ref{thm:rank} require disjoint marginal
supports across tasks on each input dimension, a condition most naturally
satisfied in domain-incremental settings (e.g., Permuted- or Rotation-MNIST).
Our main experiments use split-class CIFAR benchmarks, where this condition
holds only approximately in the learned feature space.
The density bound (Corollary~\ref{cor:density}) and empirical NTK measurements
(Table~\ref{tab:ntk}) provide complementary evidence that the structural
advantage transfers to split-class settings, but a formal theoretical
treatment of the approximate-disjointness regime---e.g., via distributional
overlap bounds on backbone features---remains an open problem and is
left for future work.
(d)~Extending per-knot regularization to adaptive-grid KANs
\cite{liu2024kan2}, where knot positions evolve during training, requires
tracking importance through grid refinement and is left for future work.
(e)~\textbf{NTK non-stationarity at task boundaries.}
Recent NTK-in-CL analyses have shown that the kernel undergoes significant
non-stationarity and reactivation at task boundaries in the feature-learning
regime, which can invalidate gradient-flow-based forgetting bounds.
Our Theorem~\ref{thm:feature-regime} establishes that the \emph{rank
advantage} of the KAN head over an MLP head persists throughout training, but
does not rule out that both kernels jointly become non-stationary. Whether
the absolute forgetting bound of Theorem~\ref{thm:ntk-bound} remains tight
across task switches---rather than merely the relative KAN vs.\ MLP
comparison---is an open theoretical question.

\section{Conclusion}

We introduced \KANCL{}+\bbEWC{}, a continual learning system that pairs a
per-knot importance regularizer on a KAN classification head with standard
EWC on a convolutional backbone. The KAN head's compact-support spline
parameterization enables spatially surgical anchoring: past knowledge is
protected precisely at the knots that were active for each task, while
inactive regions remain fully plastic. A novel NTK analysis shows that this
locality induces a structural rank deficit in the cross-task NTK that
survives feature learning, yielding a provably tighter forgetting bound than
parameter-level methods. Empirically, \KANCL{}+\bbEWC{} reduces forgetting
by 88--93\% over a head-only KAN baseline and matches or exceeds accuracy
on both benchmarks, establishing a Pareto improvement on the
stability--plasticity frontier. We hope this work encourages treating the
\emph{parameterization} of network components---not just the training
algorithm---as a primary design lever in continual learning.

\paragraph{Reproducibility.}
Code, configurations, and per-run logs are released at \url{...};
all benchmarks complete in approximately 50~min on a single A40 GPU.

\bibliographystyle{plain}

\clearpage
\appendix
\section*{Supplementary Material}
\setcounter{section}{0}
\setcounter{figure}{0}
\setcounter{table}{0}
\renewcommand{\thesection}{S\arabic{section}}
\renewcommand{\thesubsection}{S\arabic{section}.\arabic{subsection}}
\renewcommand{\thefigure}{S\arabic{section}.\arabic{figure}}
\renewcommand{\thetable}{S\arabic{section}.\arabic{table}}
\makeatletter
  \@addtoreset{figure}{section}
  \@addtoreset{table}{section}
\makeatother

The supplementary material provides additional experimental detail.
All main claims are supported by the figures and tables in the main text.

\section{Pure KAN Backbone Results}\label{app:pure-kan}

Table~\ref{tab:pure-kan} reports results for pure KAN backbones (no CNN)
on six benchmarks including MNIST-class tasks. These experiments demonstrate
that per-knot regularization improves forgetting even without a convolutional
backbone, though absolute accuracy is lower than the hybrid CNN+KAN system
due to the limited representational capacity of shallow KANs on visual data.

\begin{table}[h]
\centering
\small
\caption{Average accuracy ($\uparrow$) and forgetting ($\downarrow$) for pure
KAN backbones (Task-IL, 3 seeds). \KANCL{} delivers the lowest forgetting on
every benchmark even without a CNN backbone.}
\label{tab:pure-kan}
\resizebox{\columnwidth}{!}{%
\begin{tabular}{lcccccc}
\toprule
 & \multicolumn{2}{c}{Split-MNIST 5T} & \multicolumn{2}{c}{Perm-MNIST 10T} & \multicolumn{2}{c}{Rot-MNIST 10T}\\
\cmidrule(lr){2-3}\cmidrule(lr){4-5}\cmidrule(lr){6-7}
Method & ACC & FGT & ACC & FGT & ACC & FGT \\
\midrule
MLP finetune  & 0.760 & 0.293 & 0.322 & 0.716 & 0.456 & 0.576 \\
KAN finetune  & 0.695 & 0.365 & 0.425 & 0.569 & 0.403 & 0.609 \\
MLP+EWC       & 0.984 & 0.011 & 0.662 & 0.170 & 0.574 & 0.381 \\
KAN+EWC       & 0.926 & 0.075 & 0.555 & 0.130 & 0.504 & 0.232 \\
MLP+SI        & \textbf{0.987} & 0.005 & 0.565 & 0.156 & \textbf{0.570} & 0.296 \\
KAN+SI        & 0.968 & 0.023 & \textbf{0.578} & 0.133 & 0.524 & 0.233 \\
\textbf{\KANCL{}} & 0.924 & \textbf{0.006} & 0.524 & \textbf{0.073} & 0.520 & \textbf{0.134} \\
\bottomrule
\end{tabular}
}
\end{table}

\section{Per-Task Accuracy Trajectories}\label{app:rtraj}

\begin{figure}[H]\centering
\includegraphics[width=\columnwidth]{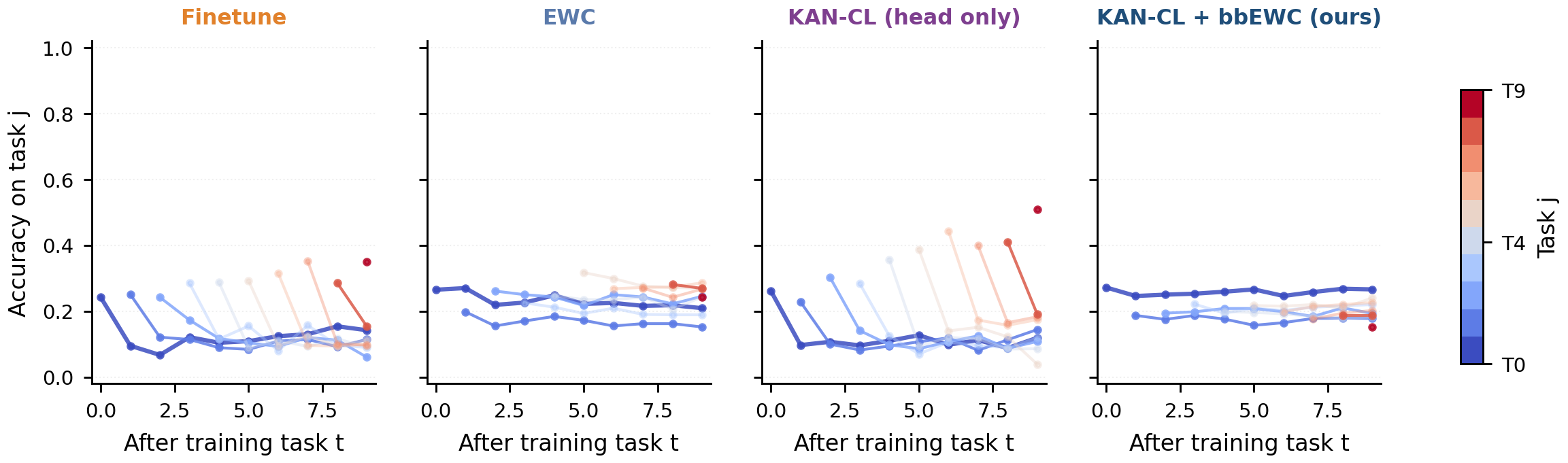}
\caption{\textbf{$R[t,j]$ trajectories on Split-CIFAR-100/10T.}
Each panel shows one method; each colored curve tracks accuracy on
evaluation task $j$ as training progresses through tasks.
\KANCL{}+\bbEWC{} (rightmost) shows the flattest trajectory profiles,
indicating minimal forgetting across all evaluation tasks.}
\label{fig:rtraj}
\end{figure}

\begin{figure}[H]\centering
\includegraphics[width=\columnwidth]{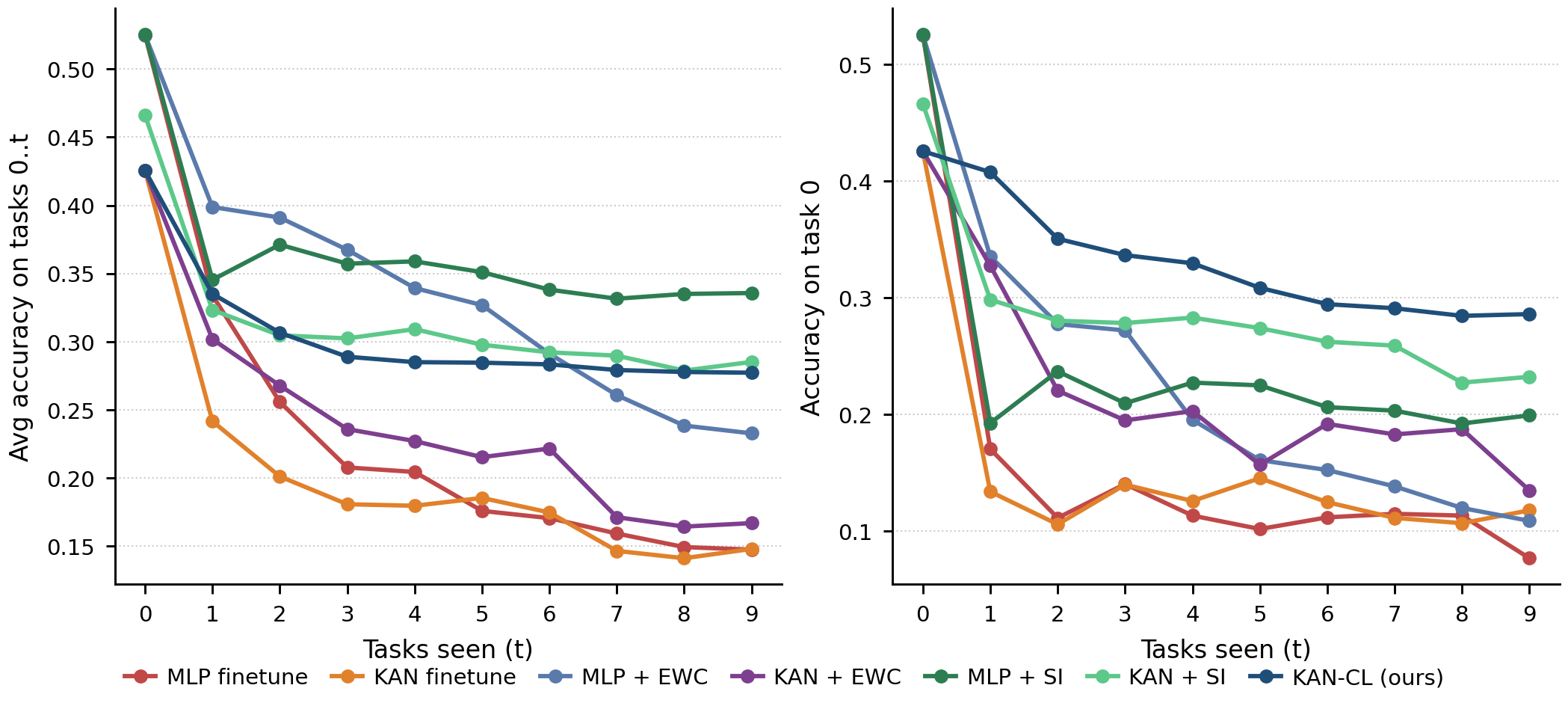}
\caption{\textbf{Mean accuracy over time on Split-CIFAR-100/10T.}
Left: mean accuracy over tasks $0..t$ after training task $t$;
right: accuracy on task~0 (first task) over training. \KANCL{}+\bbEWC{}
retains first-task performance throughout the sequence.}
\label{fig:per-task}
\end{figure}

\section{Domain-Incremental MNIST}\label{app:domain-il}

\begin{figure}[H]\centering
\includegraphics[width=\columnwidth]{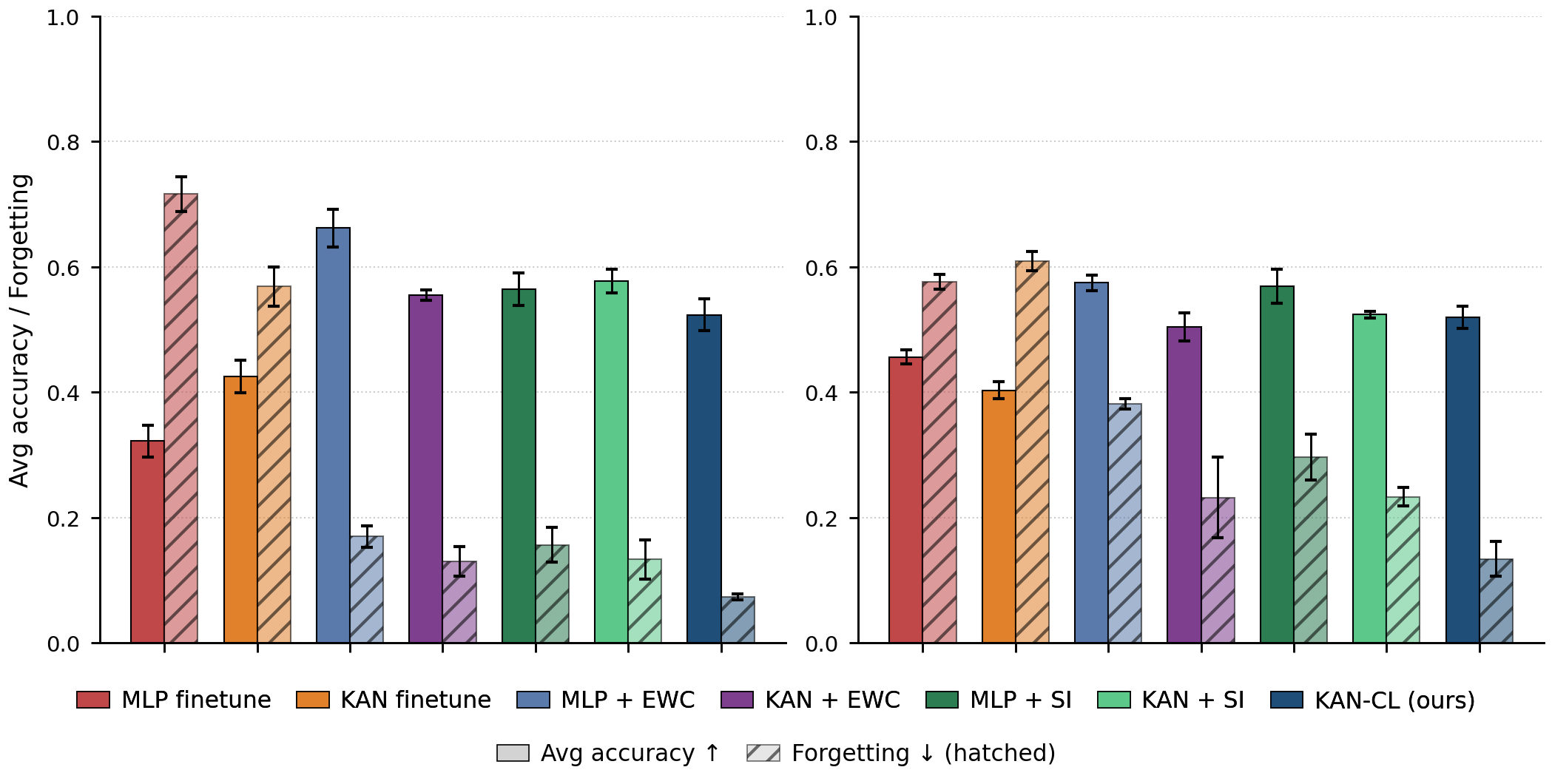}
\caption{\textbf{Domain-incremental MNIST (Task-IL, 10 tasks each).}
Left: Permuted-MNIST; right: Rotation-MNIST. Solid bars: average accuracy;
hatched bars: forgetting for all seven methods. Pure \KANCL{} achieves the
lowest forgetting while remaining competitive in accuracy.}
\label{fig:domain-il}
\end{figure}

\section{Stability--Plasticity Frontier}\label{app:pareto}

\begin{figure}[H]\centering
\includegraphics[width=\columnwidth]{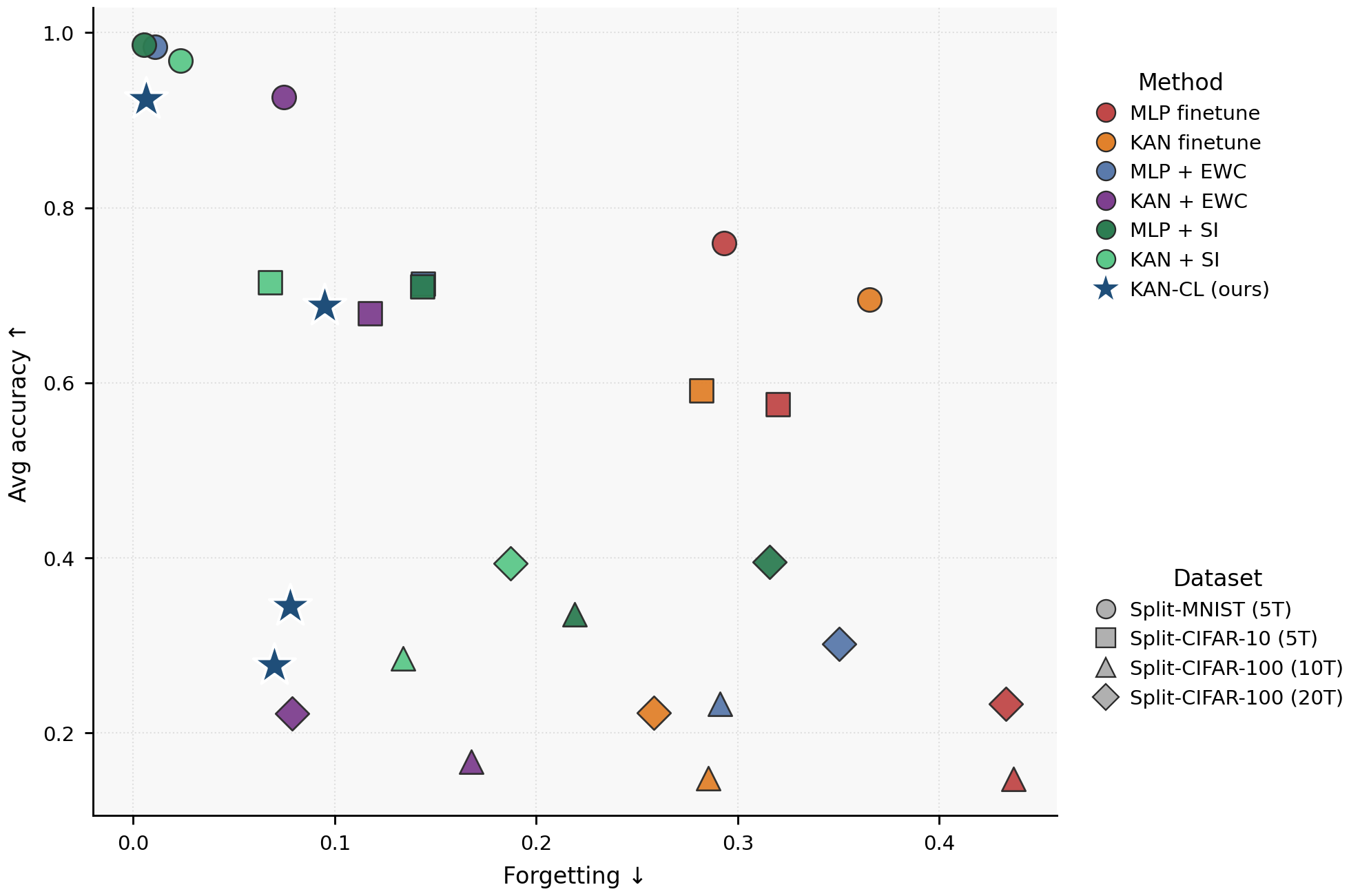}
\caption{\textbf{Stability--plasticity scatter across split-class benchmarks.}
Forgetting ($x$-axis) versus average accuracy ($y$-axis). Marker shape encodes
dataset; color encodes method. \KANCL{}+\bbEWC{} consistently occupies the
upper-left region.}
\label{fig:pareto}
\end{figure}

\section{Fisher Gradient Overlap}\label{app:fisher}

\begin{figure}[H]\centering
\includegraphics[width=\columnwidth]{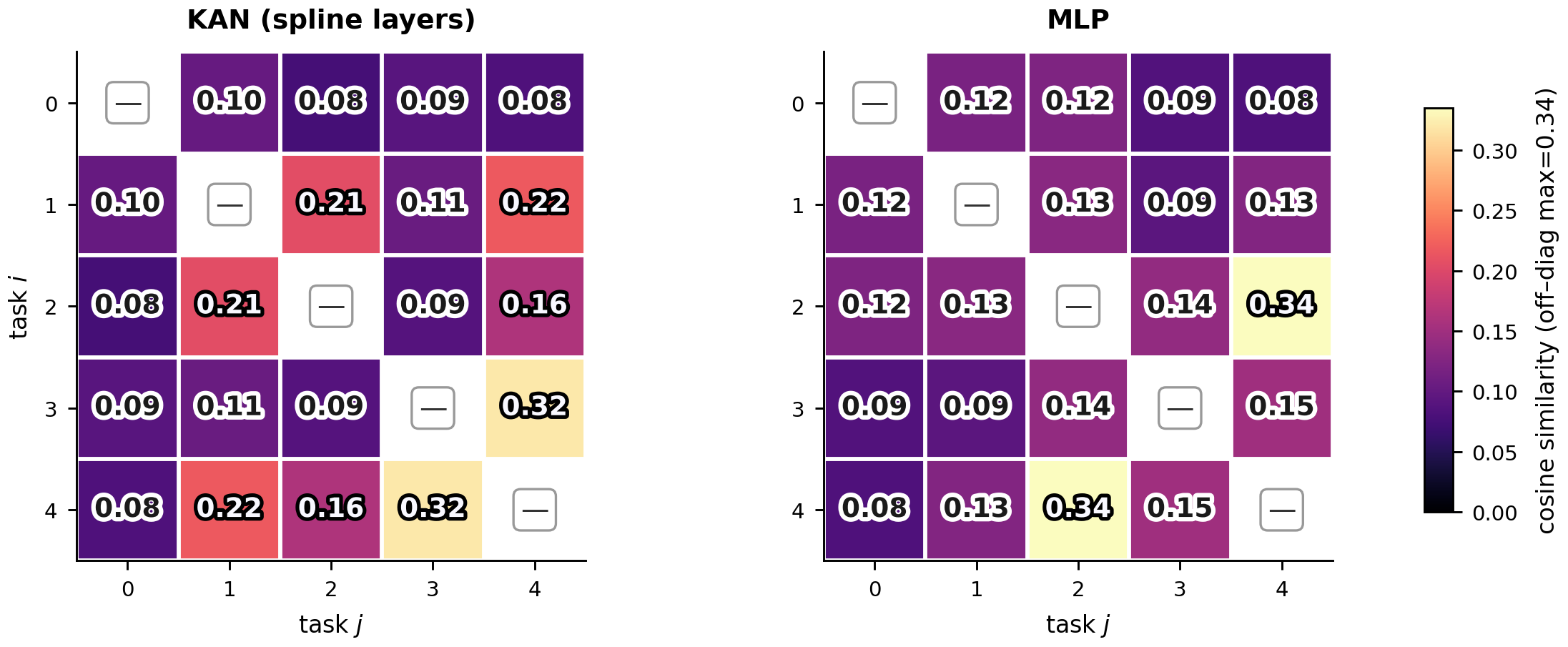}
\caption{\textbf{Task-pairwise Fisher overlap (Split-MNIST/5T).}
Cosine similarity of task-$i$ versus task-$j$ Fisher gradients (diagonal
omitted). Left: KAN spline layers; right: parameter-matched MLP. KAN exhibits
a sparser overlap structure, consistent with the per-knot disjointness of
Theorem~\ref{thm:disjoint}.}
\label{fig:fisher-ol}
\end{figure}

\section{Class-IL Buffer Size Sweep}\label{app:replay}

\begin{figure}[H]\centering
\includegraphics[width=\columnwidth]{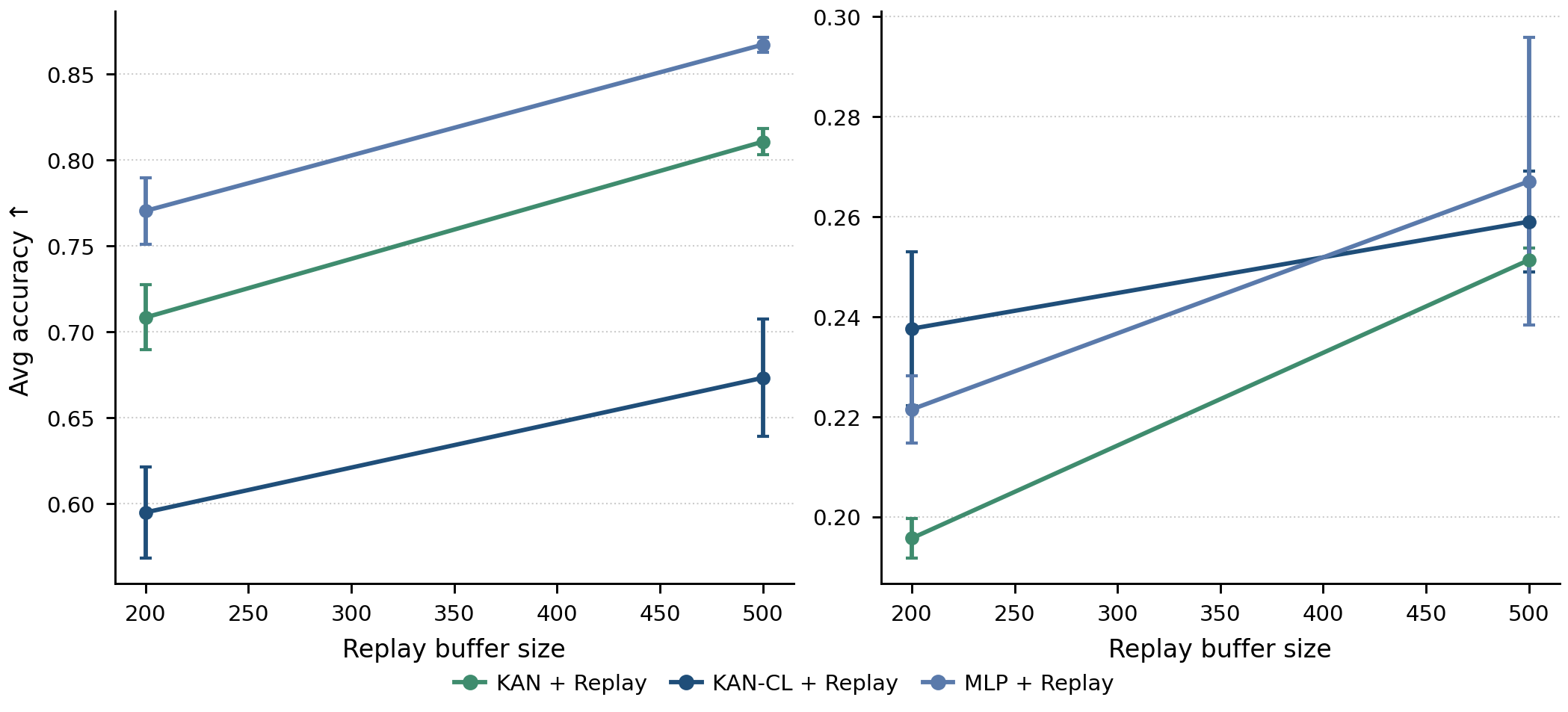}
\caption{\textbf{Class-IL buffer size sweep.}
Left: Split-MNIST; right: Split-CIFAR-10. Average accuracy vs.\ replay buffer
size for KAN+replay, MLP+replay, and \KANCL{}+replay.}
\label{fig:replay-classil}
\end{figure}

\section{Hyperparameter Sensitivity}\label{app:sensitivity}

\begin{figure}[H]\centering
\includegraphics[width=\columnwidth]{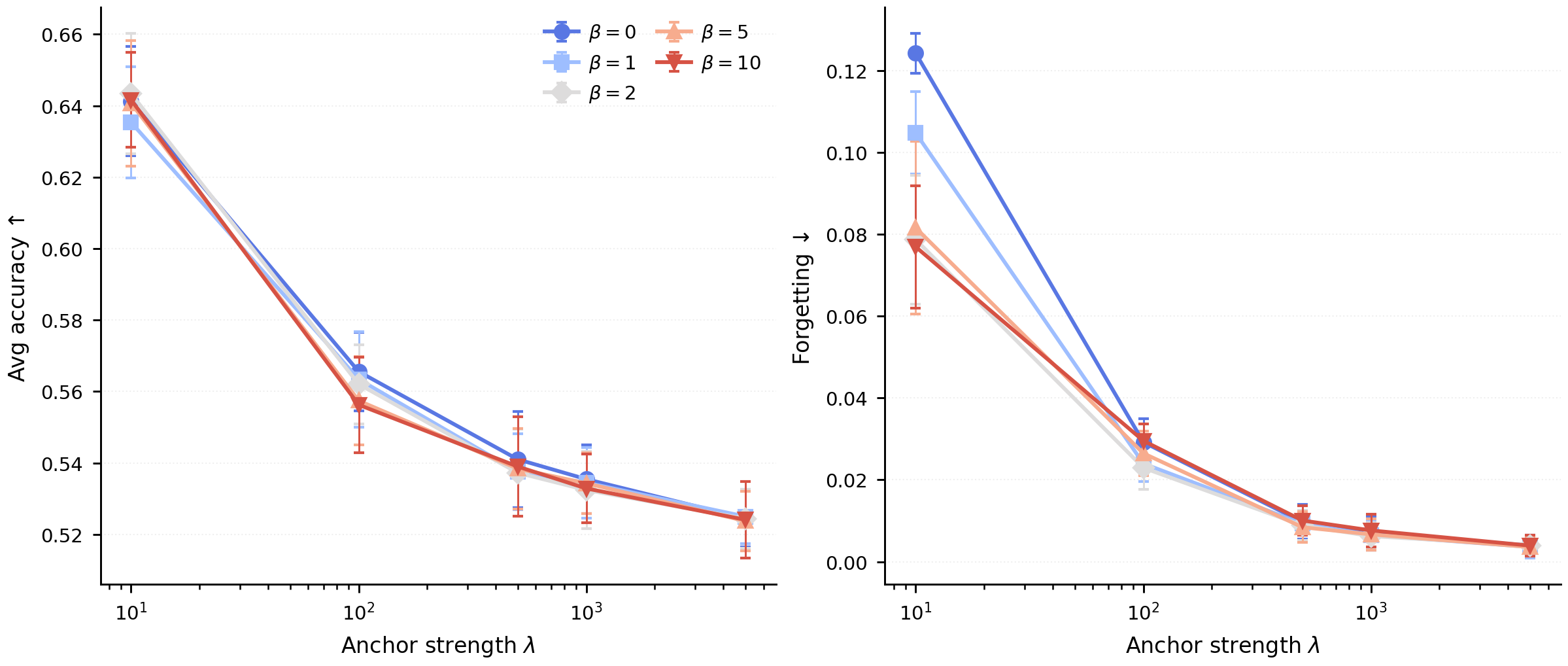}
\caption{\textbf{Sensitivity to $(\beta,\lambda)$ on Permuted-MNIST/5T.}
Left: accuracy vs.\ $\lambda$ (log scale); right: forgetting vs.\ $\lambda$.
$\lambda{=}500$ is a robust operating point across all seeds and $\beta$ values.}
\label{fig:sens}
\end{figure}

\section{Grid Size Sweep}\label{app:grid}

\begin{figure}[H]\centering
\includegraphics[width=\columnwidth]{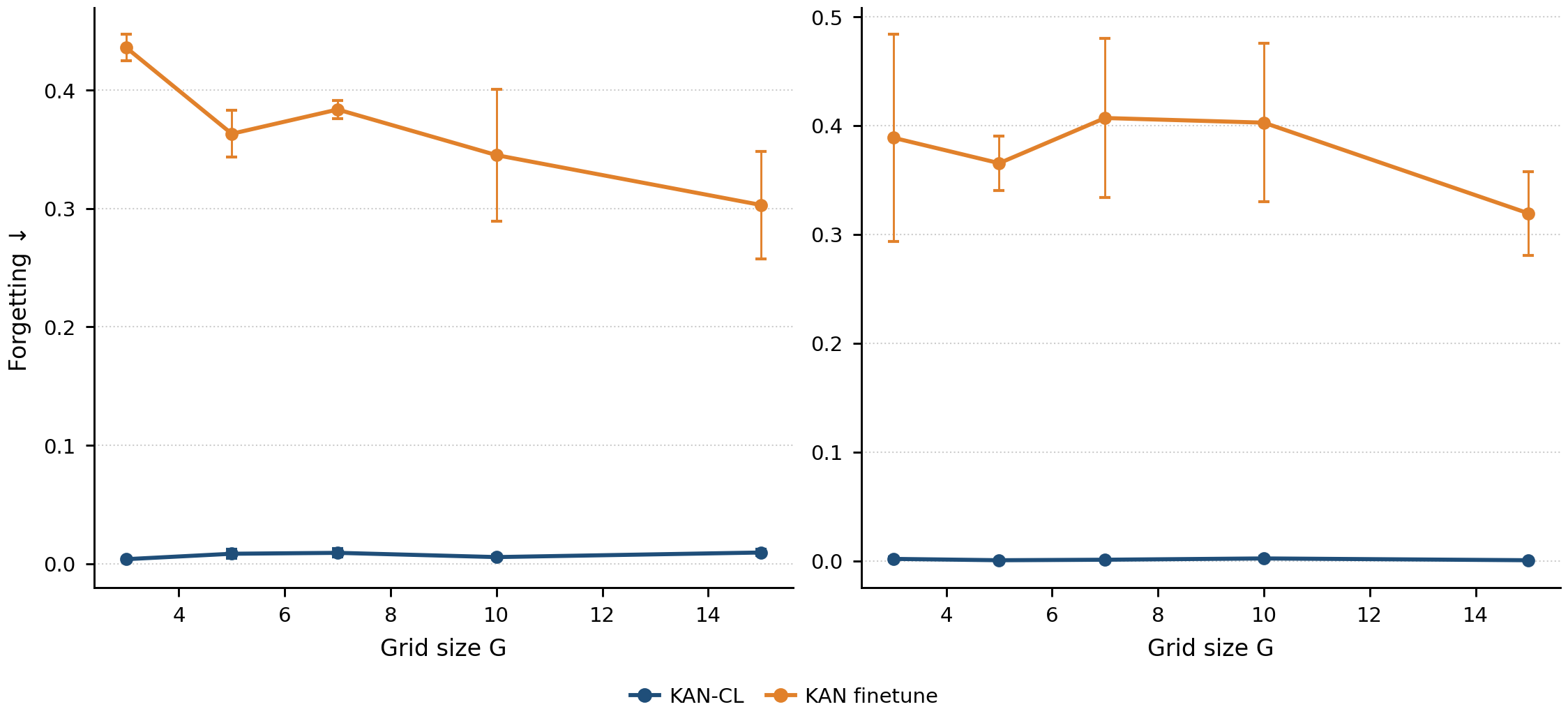}
\caption{\textbf{Grid-size sweep ($G\in\{3,5,7,10,15\}$).}
On Permuted-MNIST (domain-shift), forgetting decreases with $G$ until
variance dominates at $G^*\!\approx\!7{-}10$. On Split-MNIST (shared
marginals) forgetting is flat, matching our theoretical prediction.}
\label{fig:grid}
\end{figure}

\end{document}